\documentclass[letterpaper]{article} 
\usepackage{aaai24}  
\usepackage{times}  
\usepackage{helvet}  
\usepackage{courier}  
\usepackage[hyphens]{url}  
\usepackage{graphicx} 
\usepackage{natbib}  
\usepackage{caption} 
\usepackage{algorithm}
\usepackage{algorithmic}
\usepackage{amsmath}
\usepackage{amsfonts}
\usepackage{amsthm}
\usepackage{newfloat}
\usepackage{listings}
\usepackage{tikz}
\usetikzlibrary{positioning, fit, calc, matrix, arrows.meta}

\usepackage{xspace}
\usepackage{color}

\usepackage{subfigure}
\usepackage{algorithmic}
\usepackage{algorithm}
\usepackage{multirow}

\newtheorem{theorem}{Theorem}[section] 
\newtheorem{lemma}[theorem]{Lemma}

\newtheorem{definition}{Definition}

\DeclareCaptionStyle{ruled}{labelfont=normalfont,labelsep=colon,strut=off} 
\floatstyle{ruled}
\newfloat{listing}{tb}{lst}{}
\floatname{listing}{Listing}
\long\def\remove#1{}

\newcommand{\NNSteiner}   {{NN-Steiner}}

\newcommand {\mm}[1] {\ifmmode{#1}\else{\mbox{\(#1\)}}\fi}
\newcommand{\denselist}{\itemsep 0pt\parsep=1pt\partopsep 0pt}
\newcommand{\reals}   {{\mathbb{R}}}
\newcommand{\RSMT}    {{RSMT}\xspace}
\newcommand{\eps}     {{\varepsilon}}

\newcommand{\AroraAlg}  {{Arora's PTAS}}

\newcommand{\acell}    {{\mathsf{A}}}
\newcommand{\bcell}    {{\mathsf{B}}}
\newcommand{\rcell}    {{\mathsf{R}}}
\newcommand{\acost}    {{\mathsf{cost}}}
\newcommand{\costvec} {{\vec{\mathsf{C}}}}
\newcommand{\Nbase}    {\mathsf{NN_{base}}}
\newcommand{\NDP}       {\mathsf{NN_{DP}}}
\newcommand{\Ntop}      {\mathsf{NN_{top}}}
\newcommand{\Nretrieve} {\mathsf{NN_{retrieve}}}
\newcommand{\GNbase}    {\mathsf{NN_{base}}}
\newcommand{\GNDP}       {\mathsf{NN_{DP}}}
\newcommand{\GNtop}      {\mathsf{NN_{top}}}
\newcommand{\GNretrieve} {\mathsf{NN_{retrieve}}}
\newcommand{\myparagraph}[1]  { \textit{#1.}}
\newcommand{\myk}        {{\mathsf{k}_b}}
\newcommand{\myd}       {{\mathsf{d}_c}}

\newcommand{\II}    
{{\mathrm{\uppercase\expandafter{\romannumeral2}}}}

\newcommand{\new}[1]		{#1}

\title{NN-Steiner: A Mixed Neural-algorithmic Approach for the Rectilinear Steiner Minimum Tree Problem}
\author{
  Andrew B.\ Kahng\textsuperscript{\rm 1, 2}, Robert R.\ Nerem\textsuperscript{\rm 3}, Yusu Wang\textsuperscript{\rm 3}, and  Chien-Yi Yang\textsuperscript{\rm 2}
}
\affiliations{
     \textsuperscript{\rm 1}Department of Electrical and Computer Engineering, University of California San Diego\\
    \textsuperscript{\rm 2}Department of Computer Science and Engineering, University of California San Diego\\
   \textsuperscript{\rm 3}Halıcıoğlu Data Science Institute, University of California San Diego\\

\rm abk@ucsd.edu, rrnerem@ucsd.edu, yusuwang@ucsd.edu, chy036@ucsd.edu
}

\begin{document}

\maketitle
\frenchspacing
\begin{abstract}
Recent years have witnessed rapid advances in the use of neural networks to solve combinatorial optimization problems. Nevertheless, designing the ``right'' neural model that can effectively handle a given optimization problem can be challenging, and often there is no theoretical understanding or justification of the resulting neural model. 
In this paper, we focus on the rectilinear Steiner minimum tree (RSMT) problem, which is of critical importance in IC layout design and as a result has attracted numerous heuristic approaches in the VLSI literature. Our contributions are two-fold. On the methodology front, we propose \NNSteiner{}, which is a novel {\bf mixed neural-algorithmic framework} for computing RSMTs that leverages the celebrated PTAS algorithmic framework of Arora to solve this problem (and other geometric optimization problems). Our \NNSteiner{} replaces key algorithmic components within Arora's PTAS by suitable neural components. In particular, \NNSteiner{} only needs {\bf four} neural network (NN) components that are called repeatedly within an algorithmic framework. Crucially, each of the four NN components is only of {\bf bounded size independent of input size}, and thus easy to train. Furthermore, as the NN component is learning a generic algorithmic step, once learned, the resulting mixed neural-algorithmic framework generalizes to much larger instances not seen in training. Our \NNSteiner{}, to our best knowledge, is the first neural architecture
of bounded size that has capacity to approximately solve \RSMT{} (and variants). 
On the empirical front, we show how \NNSteiner{} can be implemented and demonstrate the effectiveness of our resulting approach, especially in terms of generalization, by comparing with state-of-the-art methods (both neural and non-neural based). 

\end{abstract}
\interfootnotelinepenalty=10000
\section{Introduction}
Given a set of points $V$ in $\reals^d$, a Steiner tree spanning $V$ 
is a tree whose vertex set is $V$ together with a set of additional 
points $S \subset \reals^d$ called {\it Steiner points}. 
A {\it rectilinear Steiner tree} is a Steiner tree where all edges 
are axis-parallel. Given $V$, the {\it rectilinear Steiner minimum 
tree (\RSMT{})} problem aims to compute the rectilinear Steiner tree 
spanning $V$ with smallest possible cost, defined as the total length 
of all edges in the tree. The RSMT problem has fundamental importance in VLSI 
physical design, as minimum wiring is correlated with key figures of 
merit including dynamic power, congestion, and timing delay. 
Hence, RSMT constructions have been well-studied for interconnect 
planning and estimation, timing estimation, global routing, and 
other applications \cite{KLMH}. 

The existence of an optimal RSMT whose Steiner points are restricted to the {\em Hanan grid}, which is formed by intersections of all axis-parallel lines passing through points in $V$, was established in \cite{Hanan66}. The \RSMT{} problem was subsequently shown to be NP-complete \cite{GareyJ77}. 
 It was proved by \cite{Hwang76} that the rectilinear minimum spanning tree gives a $3/2$-approximation of the RSMT. 
A series of results leveraging Zelikovsky's method 
led to a $5/4$-approximation \cite{BermanFKKZ94}. 
Theoretically, the best known approximation algorithm for \RSMT{} 
in fixed-dimensional Euclidean space is obtained via the PTAS 
(polynomial-time approximation scheme) proposed by 
Arora \cite{arora1998polynomial}. Arora's method
provides a $(1+\eps)$-approximation for a range 
of problems, such as the traveling salesperson problem,
in addition to \RSMT{}. Unfortunately, while this algorithm runs in 
time polynomial in the number of points, its time complexity depends exponentially on 
$\frac{1}{\eps}$ and, consequently, the method has not yet found its way to practice. 

On one hand, given the importance of the RSMT problem in chip design and its intractability, a large number of heuristics have been developed in the VLSI CAD community, aiming to improve the quality of RSMT computation with practical running time, e.g., \cite{Kahng2003,liu2021rest, Hu2006ACOSteinerAC, Fallin2022ASF, Wang2005ThePC, Cinel2008ADH}. The current state-of-the-art (SOTA) heuristic algorithm is FLUTE \cite{Wong2008flute, chu2007flute, Chu2005flute_old, Chu2004flute_old}.  
FLUTE constructs a lookup table encoding optimal RSMTs for all instances smaller than $10$ and then constructs RSMTs for larger pointsets by partitioning the input pointset into subsets of size $q$ or smaller, and combining the optimal trees over these subsets. 
FLUTE has been shown to be very close to optimal for small pointsets and is widely used in practice in VLSI CAD. GeoSteiner \cite{juhl2018geosteiner} is the SOTA algorithm for exact RSMT computation, and is an ILP-based approach. 



On the other hand, with recent success of deep neural networks (NNs)
in many applications, there has been a surge in use of NNs 
to help tackle combinatorial optimization problems 
\cite{bengio2021machine,khalil2017learning,li2018combinatorial,selsam2018learning,gasse2019exact,sato2019approximation},
such as traveling salesperson or other routing-related problems,
using reinforcement learning (RL) 
\cite{vinyals2015pointer,bello2017neural,deudon2018learning,prates2019learning}. 
Recently, REST \cite{liu2021rest} achieved the first 
NN-based approach for \RSMT{} by finding so-called 
rectilinear edge sequences using RL. 
\cite{chen2022reinforcement} designed an RL
framework to find obstacle-avoiding Steiner minimum trees.
Significant challenges in neural combinatorial optimization (NCO) remain. NNs are often used 
in an ad-hoc manner with limited theoretical understanding of 
the resulting framework.
It is also often not known if machine-learning pipelines have the capacity to solve a given combinatorial optimization problem,  or how network-architecture design could leverage problem structure to design more effective and efficient neural models.

One potential way to inject theoretical justification into the
design of neural approaches for combinatorial problems
 is to leverage the vast literature on approximation algorithms. 
 In particular, instead of 
using one NN to solve an optimization problem in 
an end-to-end manner, one can use neural components in a high-level algorithmic framework. An exemplary thread of works
uses NNs to learn variable-selection decisions in
branch-and-bound frameworks solving mixed integer-linear programming problems 
\cite{gasse2019exact,gupta2020hybrid,nair2020solving}. \cite{NNBaker} propose a mixed neural-algorithmic 
framework {\em NN-Baker} to solve
problems such as maximum independent set in the geometric setting 
by using Baker's technique to decompose problems 
into small instances of {\it bounded size}, and then training a 
single NN to solve these instances. 

\noindent 
\subsubsection{This Work.}
In this paper, we develop an approach
to compute RSMTs in $\reals^d$. (While we use $\reals^2$ and Manhattan
geometry, the framework
extends to $\reals^d$ for any constant value of $d$.) 
Specifically, we develop \NNSteiner{}\footnote{Code is open-sourced at https://github.com/ABKGroup/NN-Steiner.}, a mixed neural-algorithmic framework that leverages the  ideas 
behind \AroraAlg{} for \RSMT{} \cite{arora1998polynomial}. 

At a high level, \AroraAlg{} partitions the input domain
in a hierarchical manner, then solves the problem via a bottom-up dynamic programming (DP) procedure. 
One key result of \cite{arora1998polynomial} is that each DP subproblem
is of only bounded size, and thus, can be solved in time independent of the number of input points.
However, the complexity of this DP step is prohibitive in practice. 

In Section \ref{sec:instantiation}, we develop a mixed neural-algorithmic approach to simulate \AroraAlg{}; Fig.\ \ref{fig:nndp} gives a high-level illustration. 
The costly DP step 
is replaced by a single NN component that outputs a 
learned embedding of the solutions to the DP subproblems. This {\bf single NN module} is called repeatedly within the algorithmic framework. Other NN
components simulate the backward retrieval of Steiner points in a 
top-down manner. 
As each of the four NN components is of size independent of input problem size, the model complexity of each component is bounded.

On the theoretical front, we show in Theorem \ref{thm:theory} that this framework has the capacity to produce approximate RSMTs using only NNs of bounded complexity. 
On the practical front, NNs replace a key but costly component in \AroraAlg{}, leading to an 
efficient architecture. Furthermore, since the neural component only needs to handle fixed-size instances, training is easier. 
Once trained, \NNSteiner{} generalizes well to problems of much larger sizes than seen in training, as we demonstrate in Section \ref{sec:steinerexp}. Indeed, extensive  experimental results show that \NNSteiner{} achieves better performance 
than NN-based and non-NN-based SOTA methods for sufficiently large problem sizes.
\NNSteiner{} outperforms the existing RL-based neural approach significantly for large pointsets. \new{As input size increases, the RL-based policy must handle a larger action space, which makes it challenging both to learn the policy and to generalize.} \NNSteiner{}  learns an algorithmic component of fixed size, leading to superior generalization.

In summary, we propose \NNSteiner{}, a novel neural-algorithmic framework
for the \RSMT{} problem, which leverages the algorithmic idea
of \AroraAlg{} (which is the theoretical best approximation algorithm
for this problem). \NNSteiner{} is, to our best knowledge, the first neural architecture
of bounded size that has capacity to approximately solve the \RSMT{} problem.
Moreover, the algorithmic alignment of \NNSteiner{} leads to better
practical performance than existing SOTA methods for large instances. 
While we focus on the RSMT problem in this paper due to its practical importance in VLSI, the versatility of Arora's framework means our methodology can be potentially applied to other geometric optimization problems, such as the obstacle-avoiding RSMT  problem \cite{FOARS}. 

 Our work is one of the first NCO frameworks to use algorithmic alignment to remove dependence on problem size. To our best knowledge, the only other work in this direction is NN-Baker \cite{NNBaker}, which is limited to a very simple algorithmic setup (a flat partitioning of the input domain). The dynamic programming (DP) framework we consider is much more general: for example, a similar DP framework exists for many optimization problems (e.g., max independent set) for graphs with bounded tree width \cite{ParameterizedA}.

Removing problem-size dependence is important, as size generalization is a fundamental obstacle in NCO \cite{Garmendia2022} that is challenging to overcome \cite{Liu2022HowGI}. Training on large instances is prohibitively expensive: for supervised learning this requires computation of exact solutions to large instances, and for RL and unsupervised learning, training becomes exponentially more challenging as size increases.  Thus, size generalization is essential for performance on large instances. In some cases, such generalization is even provably hard  \cite{Yehudai2020FromLS}.
Notably, our experiments show \NNSteiner{} exhibits  strong size generalization on a hard optimization problem that has practical implications.

\section{Preliminaries}\label{sec:preliminaries}
We now introduce the \RSMT{} problem, then briefly describe 
\AroraAlg{} for this problem \cite{arora1998polynomial}. 
For simplicity, 
we henceforth treat the case where input points lie 
in $\reals^2$.
Our definitions and Arora's algorithm 
can both be extended to $\reals^d$, as well as to the standard Euclidean Steiner minimum tree problem (without rectilinear constraints). 


\begin{definition}[\RSMT{}]
Given a set of points $V \subset \reals^2$, the rectilinear Steiner minimum tree 
(\RSMT) for $V$ is a tree $T$ with vertex set $V \cup S \subset \reals^2$ with minimum 
total edge length under the $\ell_1$ norm. The set $S$ is the set of
{\it Steiner points}. 
\end{definition}
\begin{figure}[htbp]
\centering
\begin{tabular}{ccc}
\includegraphics[scale=0.35]{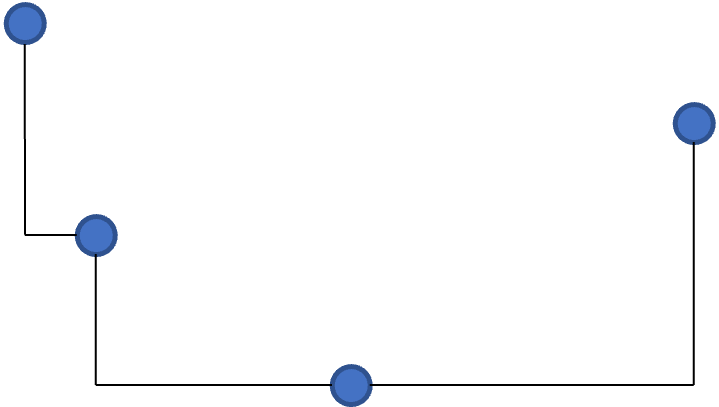} & \hspace*{0.2in}&
\includegraphics[scale=0.35]{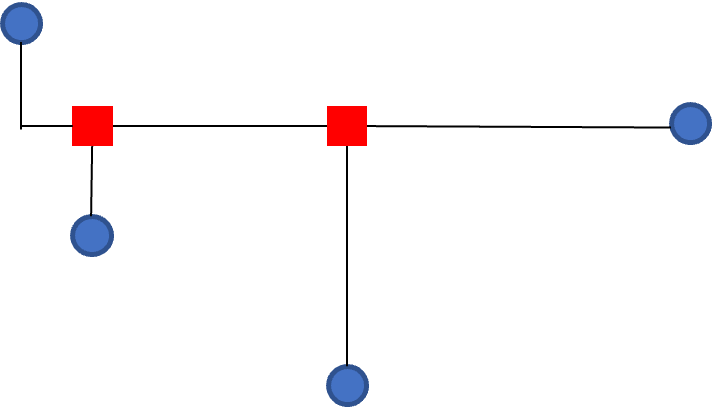} 
\end{tabular}
\caption{\small (L) Rectilinear spanning tree of input (blue) points. (R) 
Rectilinear Steiner minimum tree (red points are Steiner points).} 
\label{fig:rsmt}
\end{figure}

\subsection{Arora's PTAS}
We now describe the high-level idea behind Arora's 
polynomial-time approximation algorithm, which we refer to 
as {\it \AroraAlg{}}. 
For simplicity, we assume that the input points $V \subset \reals^2$ have integral 
coordinates, and are contained in a bounding box of side length $L = O(n)$ with 
$n = |V|$. A perturbation process given in \cite{arora1998polynomial} rounds 
input coordinates so that this assumption holds without changing theoretical 
guarantees for the algorithm. 

\subsubsection{Step 1: Construct a shifted quadtree.} 
First, we pick integers $a,b\in [0,L)$ uniformly at random
and then translate the input pointset by the
vector $(a,b)$. We then construct a quadtree, where the splitting of quadtree cells terminates if a cell contains 1 or 0 points. The 
quadtree is a tree $Q$ where each internal vertex has degree 4. The root 
has level $0$, and is associated with a cell of side length $L$. Any 
vertex $v\in Q$ at level $i$ is associated with a cell
$\acell_v$ of size (i.e., side length) $\frac{L}{2^i}$. Bisecting a level-$i$ cell $\acell_v$ with a horizontal line and with a vertical line
decomposes it into four level-$(i+1)$ child cells, each of 
size $\frac{L}{2^{i+1}}$, corresponding to the four children 
of $v\in Q$. 
As the side length of the bounding box is $L = O(n)$, and points have integral coordinates, the height 
(max-level) of the quadtree is at most $O(\log L)$ and the total number of 
vertices in $Q$ (and thus the number of cells across all levels) is 
 $O(n \log L) \subseteq O(n \log n)$. 

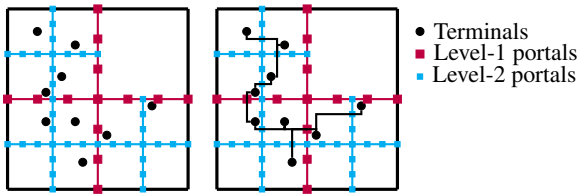
\begin{figure}[htbp]
    \centering
    \subfigure{
   \begin{tikzpicture}[scale = .6,
roundnode/.style={circle, draw=black, fill=black, very thick, minimum size=5mm},
terminal/.style={draw = black, fill = black,circle,inner sep=0pt,minimum size=3pt},
portal1/.style={draw = purple, fill = purple,rectangle ,inner sep=0pt,minimum size=3pt},
portal2/.style={draw = cyan, fill = cyan,rectangle ,inner sep=0pt,minimum size=2pt},
]

\draw[cyan,  thick] (0,1) -- (4,1);
\draw[cyan,  thick] (1,0) -- (1,4);
\draw[cyan,  thick] (0,3) -- (2,3);
\draw[cyan,  thick] (3,0) -- (3,2);

\draw[purple,  thick] (0,2) -- (4,2);
\draw[purple,  thick] (2,0) -- (2,4);

\draw[black, very thick] (0,0) -- (0,4);
\draw[black, very thick] (0,4) -- (4,4);
\draw[black, very thick] (4,4) -- (4,0);
\draw[black, very thick] (4,0) -- (0,0);

\draw  node[terminal] (1) at (0.66,3.5) {};
\draw  node[terminal] (2) at (1.5,3.2) {};
\draw  node[terminal] (3) at (1.2,2.5) {};
\draw  node[terminal] (4) at (.85,2.15) {};
\draw  node[terminal] (5) at (.85,1.5) {};
\draw  node[terminal] (6) at (1.5,1.5) {};
\draw  node[terminal] (7) at (1 + 2/3,.6) {};
\draw  node[terminal] (8) at (2.2,1.2) {};
 \draw  node[terminal](9) at (3.2,1.85) {};
 
\foreach \i in {0,...,6} 
{    \draw  node[portal1] at (2,\i*2/3) {}; 
    \draw  node[portal1] at (\i*2/3,2) {};
}

\foreach \i in {0,...,12} 
{    \draw  node[portal2] at (1,\i*1/3) {}; 
\draw  node[portal2] at (\i*1/3,1) {}; 
}

\foreach \i in {0,...,6} 
{    \draw  node[portal2] at (3,\i*1/3) {}; 
\draw  node[portal2] at (\i*1/3,3) {}; 
}

\end{tikzpicture}
    }
    \subfigure{
    \begin{tikzpicture}[scale = .6,
roundnode/.style={circle, draw=black, fill=black, very thick, minimum size=5mm},
terminal/.style={draw = black, fill = black,circle,inner sep=0pt,minimum size=3pt},
portal1/.style={draw = purple, fill = purple,rectangle ,inner sep=0pt,minimum size=3pt},
portal2/.style={draw = cyan, fill = cyan,rectangle ,inner sep=0pt,minimum size=2pt},
]

\draw  node[terminal,label=0: \small Terminals] (term key) at (4.5,3.5) {};
\draw  node[portal1,label=0: \small Level-1 portals] (port1 key) at (4.5,3) {};
\draw  node[portal2,label=0: \small Level-2 portals] (port2 key) at (4.5,2.5) {} ;


\draw[cyan,  thick] (0,1) -- (4,1);
\draw[cyan,  thick] (1,0) -- (1,4);
\draw[cyan,  thick] (0,3) -- (2,3);
\draw[cyan,  thick] (3,0) -- (3,2);

\draw[purple,  thick] (0,2) -- (4,2);
\draw[purple,  thick] (2,0) -- (2,4);

\draw[black, very thick] (0,0) -- (0,4);
\draw[black, very thick] (0,4) -- (4,4);
\draw[black, very thick] (4,4) -- (4,0);
\draw[black, very thick] (4,0) -- (0,0);

\draw  node[terminal] (1) at (0.66,3.5) {};
\draw  node[terminal] (2) at (1.5,3.2) {};
\draw  node[terminal] (3) at (1.2,2.5) {};
\draw  node[terminal] (4) at (.85,2.15) {};
\draw  node[terminal] (5) at (.85,1.5) {};
\draw  node[terminal] (6) at (1.5,1.5) {};
\draw  node[terminal] (7) at (1 + 2/3,.6) {};
\draw  node[terminal] (8) at (2.2,1.2) {};
 \draw  node[terminal](9) at (3.2,1.85) {};
 
\foreach \i in {0,...,6} 
{    \draw  node[portal1] at (2,\i*2/3) {}; 
    \draw  node[portal1] at (\i*2/3,2) {};
}

\foreach \i in {0,...,12} 
{    \draw  node[portal2] at (1,\i*1/3) {}; 
\draw  node[portal2] at (\i*1/3,1) {}; 
}

\foreach \i in {0,...,6} 
{    \draw  node[portal2] at (3,\i*1/3) {}; 
\draw  node[portal2] at (\i*1/3,3) {}; 
}

\draw[thick]  (1) |- (1, 4 - 2/3)
(1, 4 - 2/3) -| (1 + 1/3,3)
(2) -| (1 + 1/3,3)
(1 + 1/3,3) |- (3)
(3) |- (1,2 + 1/3)
(1,2 + 1/3) -| (4)
(4) -| (2/3,2)
(2/3,2) |- (5)
(5)  |- (1,1 + 1/3)
(1,1 + 1/3) -| (6)
(6) |- (2, 1+ 1/3)
(7) |- (2, 1+ 1/3)
(2, 1+ 1/3) -| (8)
 (8) |- (3, 1 + 2/3)
 (3, 1 + 2/3) -| (9)
;

\end{tikzpicture}

    }
    \caption{\small (a) A two-level quadtree over the input points (black dots), where each side of a quadtree 
    cell has 4 portals.  (b) An example of a $(2,1)$-light rectilinear Steiner tree.} 
    \label{fig:arora}
\end{figure}

We now consider a special family of Steiner trees for which the crossing of quadtree cells is constrained. 
\begin{definition}
\label{def:mr}
Let $m,r$ be positive integers. The {\it $m$-regular set of portals} 
is the set of points such that each cell has a portal at each of its 4 corners, and 
$m$ other equally-spaced portals on each of its four sides.
A Steiner tree is {\it $(m,r)$-light} if it crosses each edge of each 
cell at most $r$ times and always at a portal.
\end{definition}
 Note that if a side of a cell $S$
is contained in the sides of multiple cells, then the portals on $S$ are spaced according to the cell with side containing $S$ that has the lowest level $i$. We say that the portals on $S$ have level $i$. See Fig.\ \ref{fig:arora} for an example of quadtree decomposition and 
a $(2,1)$-light rectilinear Steiner tree.

\subsubsection{Step 2: Dynamic programming.}
The following theorem guarantees that the minimal $(m,r)$-light rectilinear Steiner tree is an approximate RSMT \cite{arora1998polynomial}. 
A proof is given in the supplemental \cite{supp}.

\begin{theorem}[Structure Theorem]
\label{th:structure}
If shifts $0 \leq a, b < L$ are chosen uniformly randomly in $[0, L)$, then with probability at least $1/2$, the minimum-length  $(m,r)$-light rectilinear Steiner  tree 
has length at 
most $\big( 1 + \frac{4}{r} + O(\frac{4\log L}{m}) \big) \mathsf{OPT}$, 
where $\mathsf{OPT}$ is the length an \RSMT. 
\end{theorem}
Hence, our goal is to compute a minimum
$(m,r)$-light rectilinear Steiner tree. In particular, Arora proposed 
to use DP in a bottom-up construction. We 
sketch the idea here. 

We process all quadtree cells in a bottom-up manner. For a fixed quadtree 
cell $\acell$, consider an $(m,r)$-light Steiner tree $T$ restricted to 
$\acell$; this gives rise to some Steiner forest $T_\acell$ in 
$\acell$, which can exit this cell only via portals on sides of $\acell$. In particular, the portion of the Steiner tree outside 
$\acell$ can be solved independently as long as we know the following 
{\it portal configuration}: (i) the set of {\it exiting portals} on the 
side of this cell that are used by $T$ (which connect points 
outside $\acell$ with those inside), and (ii) how these exiting portals 
are connected by trees in the Steiner forest $T_\acell$. 

Let $\Xi_\acell$ be the set of portal configurations for a cell 
$\acell$, and let $D = |\Xi_\acell|$.
\new{Each cell has $(4m+4)^{4r}$ subsets of $4r$ portals, and each subset of portals can be partitioned $\mathsf{Bell}(4r)$ ways. As the Bell number $\mathsf{Bell}(k) $ is bounded above by $k^k$ we have  $D =  (4m+4)^{4r}\mathsf{Bell}(4r) < 
(4m+4)^{8r}$.}
Our goal is to compute, for each portal configuration $\sigma \in \Xi$, 
the minimum cost $\acost(\sigma)$ of any rectilinear Steiner forest 
within $\acell$ that gives rise to this boundary condition. 
Assuming an arbitrary but fixed order of portal configurations 
in $\Xi = \{\sigma_1, \ldots, \sigma_D\}$, the costs of all 
configurations can then be represented by a vector 
$\costvec_\acell \in \reals^D$, where $\costvec_\acell[i] = 
\acost(\sigma_i)$. We call $\costvec_\acell$ the {\it cost-vector for 
$\acell$}.  

We now describe the DP algorithm to compute this cost-vector for all 
cells in a bottom-up manner (decreasing order of levels). 

\myparagraph{Base case: $\acell$ is a leaf cell} In this case, there is at most one point $p$ from $V$ contained in $\acell$. We enumerate all configurations, and compute $\costvec_\acell$ directly, which requires solving RSMT instances of bounded size. 

\myparagraph{Inductive step: $\acell$ is not a leaf cell} The four 
child-cells $\acell_1, \ldots, \acell_4$ of $\acell$ are from the 
level below  $\acell$'s level, and thus, by the inductive hypothesis, we have already 
computed the cost-vectors $\costvec_{\acell_i}$ for $i=1, \ldots, 4$. 
Consider any portal configuration $\sigma \in \Xi$. We simply need to 
enumerate all choices of portal configurations $\tau_1, \ldots, \tau_4$ 
for child-cells $\acell_1, \ldots, \acell_4$  that are 
\emph{consistent with $\sigma$}, meaning that the portals on common 
sides are the same, and the connected components of portals from the 4 
child-cells do not form cycles (hence, still induce a valid Steiner 
forest).  We have
\begin{equation}
\label{eq:induction}  
\acost(\sigma) = \min_{\tau_1, \ldots, \tau_4 ~\text{consistent with}~\sigma}   \sum_{i\in \{1,2,3,4\}}\acost[\tau_i]
\end{equation}



\myparagraph{Final construction of approximate RSMT}
At the end of DP, after we compute the cost-vector for the root cell, we identify the 
portal configuration $\sigma^*\in \Xi$ with lowest cost. To 
obtain the corresponding rectilinear Steiner tree,  we perform a top-down backtracking: 
(i)  for the root cell, from $\sigma^*$ we can retrieve 
the set of child-cell configurations $\tau_1^*, \ldots, 
\tau_4^*$ generating 
$\sigma^*$; (ii) we repeat this until we reach all leaf cells. At each leaf cell we once again compute the optimal Steiner forest for the chosen configuration. Combining these Steiner forests yields an 
optimal $(m,r)$-light RSMT.

\section{NN-Steiner}

Although the running time of \AroraAlg{} is polynomial in the problem size, 
its computation is prohibitively expensive in practice as we compute all 
possible portal configurations. 
Instead of brute-force enumeration of portal configurations, 
we propose a framework, {\em NN-Steiner}, which infuses neural networks 
(NNs) into \AroraAlg{}  to  select Steiner points from portals to build
the output Steiner tree. In Sec.\ \ref{sec:simulate}, we show that the key components 
within the DP algorithm can be simulated exactly by certain NNs. 
However, such NNs are not efficient either, \new{as their size depends exponentially on $m$ and $r$}. In Sec.\ \ref{sec:instantiation}, 
we show a practical instantiation of \NNSteiner{} and demonstrate its 
performance in Sec.\ \ref{sec:steinerexp}. 

\subsection{Theoretical \NNSteiner{} to simulate \AroraAlg{}}
\label{sec:simulate}

\begin{figure*}[h]
    \centering
    \includegraphics[width=1.8\columnwidth]{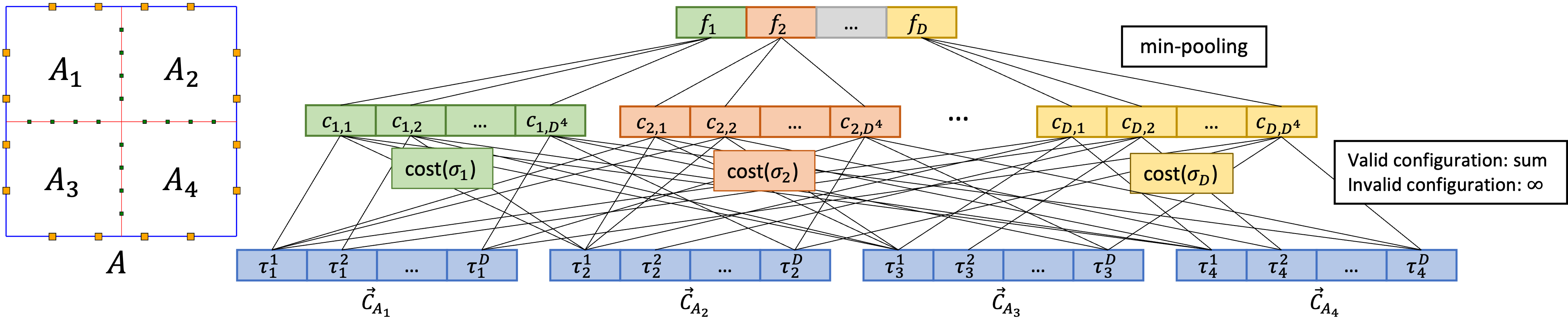}
    \caption{ \small An NN simulating the dynamic-programming function $f_{DP}$. Here, $\costvec_{\acell_i}[\sigma_j] = \tau_i^j$ encodes the
   $j^{\text{th}}$ configuration cost in cell $\acell_i$.  We sum 
    over costs $\acell_i$ $(i=1,2,3,4)$ consistent with a portal 
    configuration $\sigma_j$ $(j=1,2,...,D)$ so that the resulting vector cost($\sigma_j$) encodes 
    the cost of all configurations consistent with $\sigma_j$. 
    Min-pooling on $\acost(\sigma_j)$ yields the minimum cost 
    of configuration $\sigma_j$.} 
    \label{fig:nn}
\end{figure*}
We can simulate key components in the DP framework
of \AroraAlg{}. Specifically, we use four neural networks: 
$\Nbase$ and $\NDP$ respectively implement the base case and inductive step 
of DP to compute encodings of cost-vectors in a bottom-up manner, while
$\Ntop$ and $\Nretrieve$ obtain optimal portal configurations from cells in a top-down manner.

There exist designs and parameters of these NNs 
simulating Arora's DP algorithm exactly; moreover, these NNs are 
each of only {\bf bounded size} independent of $n$ (depending only on 
parameters $m$ and $r$, which are set to be constant in practice). 
Here, we describe how to construct 
$\NDP$ to simulate one inductive step in the DP algorithm. For brevity, we leave the
other NN cases to a full version. 

Recall that, as described in Sec.\ \ref{sec:preliminaries},  the inductive step of the DP algorithm can be 
rewritten as applying a function $f_{DP}: (\reals^D)^4 \to \reals^D$, where $D$ is the total number of portal configurations for a single cell. 
In particular, for any quadtree cell $\acell$ with child-cells 
$\acell_1, \ldots, \acell_4$, the input to $f_{DP}$ is the four 
cost-vectors $(\costvec_{\acell_1}, \costvec_{\acell_2}, 
\costvec_{\acell_3}, \costvec_{\acell_4}) \in (\reals^D)^4$, and the 
output is the cost-vector $\costvec_{\acell} \in \reals^D$. 

Eqn.\ (\ref{eq:induction}) gives how to compute each entry 
in the output vector $f_{DP}(\costvec_{\acell_1}, \ldots, \costvec_{\acell_4})$. 
That is, $f_{DP}$ has a simple form modeled as a certain linear 
function followed by a min-pooling, which can be simulated by a NN 
as shown in Fig.\ \ref{fig:nn}. In particular, to compute
$\costvec_{\acell}[\sigma_i]$, each neuron $c_\ell$ takes 
in a set of four portal configurations $\tau_j \in \acell_j$, $j = 1, 
\ldots, 4$, consistent with $\sigma_i$, and simply does a 
sum operation. Then the output takes a min-pooling over all 
values at $c_\ell$s. 
As there are at most $D^4$ sets of four 
portal configurations for each $\sigma_i$, the entire model has 
complexity $\Theta(D^5)$. (Recall that $D \le (4m+4)^{8r}$ is 
independent of the input pointset size $n$.) 
That is, there is a $\NDP$ of bounded complexity simulating
the DP step exactly.  The following theorem summarizes the
existence of NNs to implement \AroraAlg{}.

\begin{theorem}
\label{thm:theory}
There exist four NNs, each of only bounded size depending 
only on $m$ and $r$, that can simulate the DP 
algorithm of \AroraAlg{}, such that the resulting mixed neural-algorithmic framework can 
find a $\big(1 + \frac{4}{r} + O(\frac{4\log L}{m})\big)$-approximate 
rectilinear Steiner tree (i.e., with length at most 
$\big(1 + \frac{4}{r} + O(\frac{4\log L}{m})\big) \mathsf{OPT}$). 
The framework calls these NNs only $O(n \log L) \subseteq O(n \log n)$ times. 
\end{theorem}

\subsection{Practical Instantiation of \NNSteiner}
\label{sec:instantiation}

 The theoretical neural-algorithmic framework described in the previous section is not 
practical as it is explicitly encoding the exponential 
number of portal configurations (exponential in $m, r$). 
We now present \NNSteiner{}, a practical instantiation of this framework. Theorem \ref{thm:theory} implies that the \NNSteiner{} architecture has the capacity to approximately solve the RSMT problem using NN components of bounded size.  In practice, one hopes that \NNSteiner{} can leverage data to encode portal configurations more efficiently.

The pipeline starts with a neural network $\GNbase$, that acts on each leaf cell to produce an encoding of the configuration costs. Next, we apply a NN simulation of the dynamic programming step $\GNDP$, for each non-leaf cell. We then use two NNs, $\GNtop$ and $\GNretrieve$, to simulate the backtracking stage and return the likelihood that each portal is a Steiner point. By selecting high-likelihood portals, we construct a set of Steiner points and a corresponding Steiner tree. As the selected Steiner points $S$ must 
lie on cell boundaries, we finish with a local refinement scheme which introduces Steiner points that lie in the interior of leaf cells.  See Fig.\ 
\ref{fig:nndp} for a high-level overview of our practical \NNSteiner{} pipeline. 
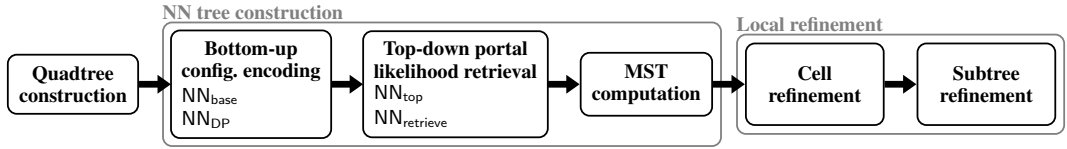
\begin{figure*}[h]
    \centering
\begin{tikzpicture}
   \node [scale=0.85, transform shape,  inner sep=0] {
   
   \begin{tikzpicture}[node distance=.5cm, font = \footnotesize,
  mymatrix/.style={matrix of nodes, nodes=typetag, row sep=-.6 em, column sep = -1.5 em},
  mycontainer/.style= {font =  \bf\footnotesize, thick, draw = black, rounded corners, inner sep=0 },
    largecontainer/.style= {font =  \bf\footnotesize, thick, draw = gray, rounded corners, inner sep=.125 cm},
  typetag/.style={ inner sep=1ex, align = left, anchor=west},
  title/.style={font =\footnotesize \bf}
  ]
    \node[mycontainer, align=center, inner sep=.5em](mx0) {Quadtree \\ construction};

  \matrix[mymatrix, right=of mx0.east, matrix anchor= west] (mx1) {
    \node[title,align=center] {Bottom-up \\config. encoding}; \\
    $\GNbase$ \\
    $\GNDP$ \\
  };
  \matrix[mymatrix, right=of mx1.east, matrix anchor=west] (mx2) {
    \node[title, align = center] {Top-down portal \\likelihood retrieval}; \\
    $\GNtop$\\
    $\GNretrieve$ \\
  };
    \node[mycontainer, align=center, inner sep=.5em,right=of mx2.east, anchor = west](mx25) {MST \\ computation}; 
  
\node[mycontainer, align=center, thick, inner sep=1em, right=of mx25.east, anchor = west](mx3) { Cell \\ refinement};
\node[mycontainer, align=center, inner sep=1em, right=of mx3.east, anchor = west](mx4) {Subtree \\ refinement};

\node[mycontainer, fit=(mx1) ]  {};
\node[mycontainer, fit=(mx2) ]  {};
\node[largecontainer, fit=(mx1) (mx2) (mx25) ](tree) {};
\node[largecontainer, fit=(mx3) (mx4)  ] (refinement){};
\node[ above =  1 pt of tree.north west, anchor = south west , text = gray  ] {\bf NN tree construction};
\node[ above = 1 pt  of refinement.north west, anchor = south west, text = gray   ] {\bf Local refinement};
\draw[line width=3pt, ->, -{Triangle[width=8pt,length=6pt]}]  (mx0.east) -- (mx1.west);
\draw[line width=3pt, ->, -{Triangle[width=8pt,length=6pt]}]  (mx1.east) -- (mx2.west); 
\draw[line width=3pt, ->, -{Triangle[width=8pt,length=6pt]}]  (mx2.east) -- (mx25.west) ;
\draw[line width=3pt, ->, -{Triangle[width=8pt,length=6pt]}]  (mx25.east) -- (mx3.west) ;
\draw[line width=3pt, ->, -{Triangle[width=8pt,length=6pt]}]  (mx3.east) -- (mx4.west) ;
\end{tikzpicture}
   
   };  
\end{tikzpicture}
    \caption{\small Pipeline of an NN-Steiner instantiation.}
    \label{fig:nndp}
\end{figure*}

\subsubsection{Forward pass.} The forward processing involves 
two MLPs, $\GNbase$ and $\GNDP$, and calls $\GNDP$ recursively in a
bottom-up manner to compute an \emph{implicit encoding of the costs of 
possible portal configurations}.

\myparagraph{Base-case neural network $\GNbase$} 
At the leaf level of \AroraAlg{}, each cell contains at most $1$ point. 
We instead terminate the quadtree decomposition when a cell 
contains no more than $\myk$ points, where $\myk$ is a hyperparameter. 
Given a leaf cell $\bcell$ with $V_\bcell \subset V$ the set of 
points contained in $\bcell$, $\GNbase$ takes relative coordinates of $V_\bcell$ as well as the set of portals on the sides
of $\bcell$ as input, and 
outputs a $\myd$-dimensional vector as an implicit 
encoding of the cost vector $\costvec_\bcell \in \reals^D$; 
note that $\myd << D$ in practice. Coordinates are specified relative to the bottom left corner of the cell, and are normalized by the cell size. \new{If a leaf cell has $n_p  < \myk$ input points, we pad the remaining $\myk - n_p$ coordinates with $(-1,-1)$.}  Each cell can have at most  $4m+8$ portals ($m$ + 2 on each side) and these portals can appear only at $4m+8$  distinct relative locations in the cell.  Portals are then provided to $\GNbase$ as $4m+8$ indicators in $\{0,1\}$.
Here, each corner contains two portals to simplify computation by distinguishing connections passing through from different sides. (Arora's PTAS uses a single portal at each corner.)

Terminating with at most $\myk$ points in each leaf cell is advantageous as, for a quadtree cell with very few points, it is challenging to learn a 
meaningful encoding of portal configurations. If $\myk$ is small, the majority of 
cells have few points inside. (Note that a complete degree-4 tree has around 
75\% of its nodes at the leaf level.) Hence, training on such a collection of 
cells tends to provide bad supervision, harming the effectiveness of learning. 
On the other hand, as $\myk$ increases, larger Steiner-tree instances must be computed at leaves, which may negatively affect the performance of the framework.
In our experiments, $\myk$ is a hyperparameter; see 
Sec.\ \ref{sec:steinerexp} for its effect and choice. 


\myparagraph{ DP-inductive-step neural network $\GNDP$} 
Next, we use another MLP, $\GNDP$, to simulate the
function $f_{DP}$ which, as mentioned in Sec.\ \ref{sec:simulate}, is 
equivalent to the DP in \AroraAlg{}. 
In detail, given a cell $\acell$ at level $i$, let $\acell_1, 
\ldots, \acell_4$ be its 4 child-cells at level $i+1$ in a fixed order. 
The neural network $\GNDP$ takes encodings of $\costvec_{\acell_1}, \dots, \costvec_{\acell_4}$, and the portals of $\acell$, and generates an encoding of 
the cost vector $\costvec_\acell \in \reals^D$ of the parent cell $\acell$. Note that the encodings of $\costvec_{\acell_1}, \dots, \costvec_{\acell_4}$ are produced by applying either $\GNDP$ or $\GNbase$  at $\acell_1, 
\ldots, \acell_4$.   Again, portals are provided via $4m+8$ indicators in $\{0,1\}$.

\subsubsection{Backward pass.} 
As described above, the forward pass applies $\GNbase$ 
to all leaf cells, and then $\GNDP$ to all internal vertices, to simulate 
the bottom-up DP algorithm of \AroraAlg{}. 
Now,
we use two more NNs to simulate the backtracking stage of 
the DP algorithm. Together, these NNs construct a \emph{portal-likelihood map}  $\rho: \mathcal{P} \to [0,1]$, where $\mathcal P$ is the set of all portals.

\myparagraph{Root-level retrieval neural network $\GNtop$} 
For a cell $\acell$ let $\mathrm{Portals}(\acell)$ denote the portals in $\acell$ that are one level higher than $\acell$'s level, i.e., the portals on the vertical and horizontal segments bisecting $\acell$. Let $\rcell$ be the root cell with children $\acell_1, 
\ldots \acell_4$. The input to $\GNtop$ is the output of $\GNDP$ at $\rcell$, and the input of $\GNDP$ at $\rcell$. 
The output of $\GNtop$ is a vector of the likelihoods of $\mathrm{Portals}(\rcell)$, where $\rcell$ is the root cell.

\myparagraph{Backward retrieval step neural network $\GNretrieve$} 
We compute the rest of the portal likelihoods in a top-down manner. This  is 
achieved by using a neural network $\GNretrieve$ (at all 
non-root non-leaf cells $\acell$) which computes the likelihoods of $\mathrm{Portals}(\acell)$. 
The input of $\GNretrieve$ applied at a cell $\acell$  with child-cells $\acell_1, 
\ldots$ and $ \acell_4$, is comprised of two parts.
First, $\GNretrieve$ takes the output of four instances of either $\GNbase$ or $\GNDP$ applied at $\acell_1, 
\ldots$, and $ \acell_4$ (same as $\GNtop$). Second, $\GNretrieve$  receives likelihoods for every portal on the boundary of $\acell$.  At this step, these likelihoods will have been computed previously by an instance of either $\GNretrieve$ or $\GNtop$ applied at the level above $\acell$'s level.

\subsubsection{Retrieval of Steiner points and postprocessing.}
After the backward pass, we have a portal-likelihood map $\rho$
over all portals. We select all portals with likelihood greater than \new{a threshold $t \in (0,1)$} as the initial set of Steiner points $S$. We then compute the minimum spanning tree over $S \cup V$, which takes time $O(|S \cup V|\log |S \cup V|)$. Next, we apply three local refinement steps:
\begin{enumerate}\denselist 
    \item First, we iterate over all leaf cells and for each cell replace every connected component with the optimal Steiner tree connecting all of the component's Steiner points (selected portals) and input points. 
    \item  Next, we remove Steiner points with degree less than 3 and round the locations of the remaining Steiner points to integer coordinates. 
    \item Finally, we partition\footnote{We partition by iterating the following procedure: (i) select a leaf, (ii) use breadth-first search to select vertices until $k$ input points or all remaining input points are selected, then (iii) remove these vertices. The subtree induced by each removed set of vertices, which contains at most $k$ input points, forms an element of the partition.} the tree into subtrees with $k$ or fewer input points. We replace each subtree with the optimal Steiner tree over both its input points and any Steiner point which is adjacent to a vertex in another partition. 
\end{enumerate}

Optimal Steiner trees are computed using GeoSteiner 5.1 \cite{juhl2018geosteiner}. The first step, cell refinement, introduces Steiner points into the interior of cells, since initially Steiner points can only be at portals.  The second step removes unnecessary Steiner points and rounds the locations of Steiner points. Rounding simplifies the tree and can be done with minimal effect on the solution since an optimal Steiner tree always lies on the Hanan grid and thus has integer Steiner points. The final step, subtree refinement, allows Steiner points at portals to be moved. 

\subsubsection{Training.}

\begin{figure}[h]
    \centering
    \includegraphics[width =\columnwidth]{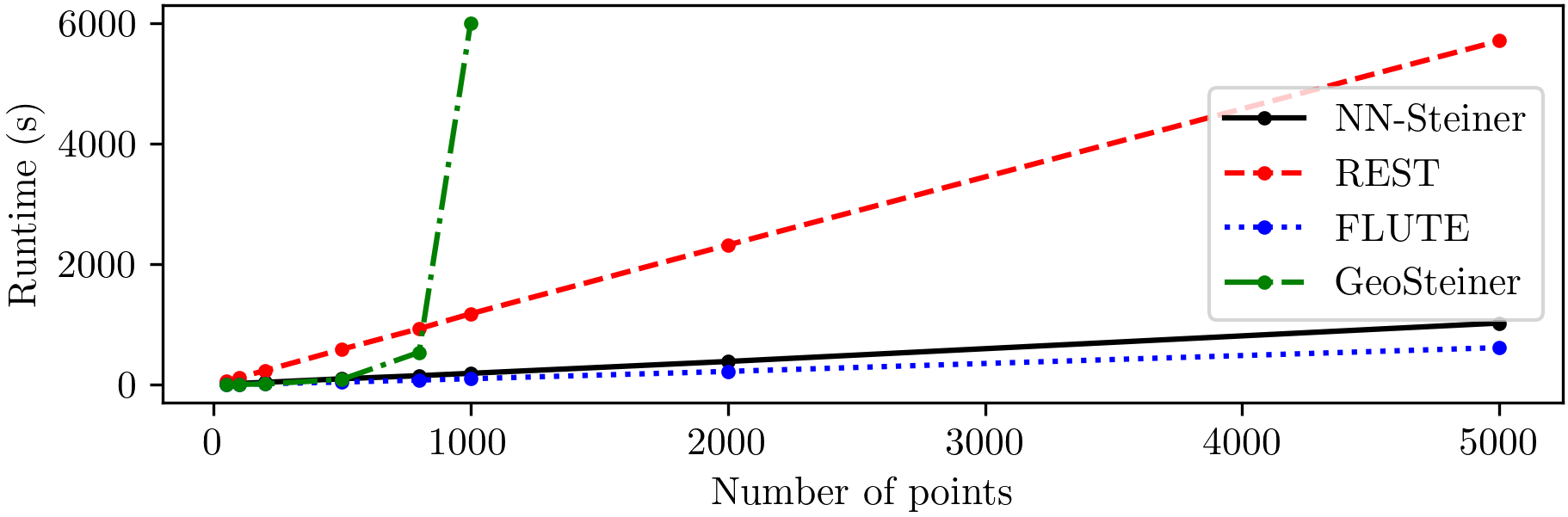}
\caption{\small Total runtime for solving 100 pointsets.}
    \label{fig:runtime}
\end{figure}

 For each training pointset, we compute the optimal Steiner tree using GeoSteiner 5.1 \cite{juhl2018geosteiner}. Each time this optimal tree crosses a side of a cell, we move the crossing to the nearest portal.  The resultant set of portals with crossings is used as the target in computing the loss. For the loss, we use binary cross entropy with weights of $m + 1$ for portals that are Steiner points in the target, and weights of 1 for other portals. As the classification is imbalanced, with most portals not being Steiner points, this weighting scheme prevents the classifier from achieving low loss with uniformly near-zero portal likelihoods. 

We train all four networks in an end-to-end manner. In training, the models are connected in a tree structure similar to that used in RNNs (recursive neural networks) \new{in that the output of $\NDP$ is fed into the same $\NDP$ instance repeatedly in our architecture. However, RNNs are often connected linearly while \NNSteiner{} connects NNs in a tree structure.} To accelerate training, we utilize batch-mode training. Batch mode necessitates that the same tree structure is used for each training sample. Otherwise, the model shape would be different between training samples, and these samples could not be learned in parallel. 

We use pointsets sampled from a distribution tailored for compatibility with batch-mode training. Pointsets are constructed as follows: (i)  Sample $n_0$ points from a distribution $\mathcal D$ on an integer grid of size $N \times N$. (ii) Construct a depth-$d$ quadtree. (iii) In each quadtree leaf cell with greater than $k_b$ points, remove points randomly until only $k_b$ points remain. 
With these pointsets, a depth-$d$ quadtree can always be used with no more than $k_b$ points in every leaf cell. 
Note that while we use a fixed tree structure throughout training, for testing, the tree structure is decided by the specific test pointset and can have any shape or depth.

\section{Experimental performance}
\begin{table*}[!htbp]
\centering
\small
\begin{tabular}{|l|r|r|r|r|r|r|r|r|}
\hline
Number of points & 50             & 100           & 200           & 500            & 800            & 1000           & 2000                  & 5000                  \\ \hline
NN-Steiner       & 2.10	          &  1.38	      & 0.74	      & \textbf{-0.67}	& \textbf{-1.11} & \textbf{-1.43} & \textbf{-2.44}        & \textbf{-2.99}         \\ \hline
REST (T=8)        & \textbf{-0.14} & 1.08          & 7.40          & 22.68          & 35.16          & 42.52          & 75.12                 & 147.48                \\ \hline
FLUTE (A=18)      & 0.00           & \textbf{0.00} & \textbf{0.00} & 0.00           & 0.00           & 0.00           & 0.00                  & 0.00                  \\ \hline
GeoSteiner       & -0.54          & -1.23         & -2.25         & -3.71          & -4.43          & -4.78          & -- & -- \\ \hline
\end{tabular}
\caption{\small Performance comparison on uniformly distributed pointsets (average percent length difference compared to FLUTE).   }
\label{tab: uniform-length}
\end{table*}
\begin{table*}[!htbp]
\small
\centering
\begin{tabular}{|l|l|l|l|l|l|l|l|l|}
\hline
Number of points                                              & 50    & 100   & 200   & 500   & 800   & 1000  & 2000  & 5000  \\ \hline
NN tree construction + refinement (\NNSteiner{}) & 2.10  & 1.38  & 0.74  & -0.67 & -1.11 & -1.43 & -2.44 & -2.99 \\ \hline
MST of $V$ + refinement                                       & 2.54  & 2.50  & 1.44  & 0.04  & -0.75 & -1.04 & -1.73 & -2.46 \\ \hline
NN tree construction + cell refinement                        & 3.01  & 2.54  & 1.84  & 0.39  & -0.10 & -0.45 & -1.39 & -2.08 \\ \hline
NN tree construction + subtree refinement                     & 6.76  & 5.55  & 4.90  & 3.39  & 3.20  & 3.07  & 1.64  & 1.44  \\ \hline
FLUTE + refinement                                            & -0.03 & -0.08 & -0.15 & -0.22 & -0.27 & -0.27 & -0.31 & -0.33 \\ \hline
\end{tabular}
\caption{\small Ablation experiments showing the effect of refinement (average percent length difference compared to FLUTE).}
\label{tab: uniform-ablation}
\end{table*}

\label{sec:steinerexp}
We present experimental results of \NNSteiner{} on 
planar RSMTs. Our experiments show that \NNSteiner{} outperforms SOTA algorithms for large point sets and has approximately linear runtime.  We demonstrate the effectiveness of \NNSteiner{} on differently distributed pointsets. We give ablations that elucidate the role of our portal-retrieval and refinement schemes and we show the dependence of \NNSteiner{} on critical hyperparameters. 
All of our experiments run on a 64-bit Linux server with a 2.25GHz AMD EPYC 7742 Processor (256 threads) and three Nvidia RTX A100-SXM4 GPUs allocating 80GB RAM.

\myparagraph{Implementation}
Each of $\Nbase, \NDP, \Ntop, \Nretrieve$ is implemented using a 4-layer MLP, with ReLU activation and size-4096 hidden layers.
The output of $\Nbase$ and $\NDP$ is a vector of dimension $\myd = 16(4m +8)$.
An additional sigmoid layer is appended to $\Ntop$ and $\Nretrieve$ to output the portal likelihoods in $[0,1]$. For training, we use a dropout of $0.1$ and a batch size of 5000.  Training pointsets are generated using a uniform distribution with $n_0 = 180$ initial points, $N = 100$ grid lines, and using a quadtree of depth $d = 3$. ( In testing, the quadtree may have depth much greater than $d$, depending on the number and distribution of points.)
We train the models on 120,000 pointsets solved exactly by GeoSteiner. For the refinement step, we partition into subtrees of size $k = 10$.  We train for 5000 epochs using Adam \cite{kingma2017adam} as our optimizer, with a learning rate of $10^{-4}$. The default setting of \NNSteiner{} is $m = 15$ and $k_b = 4$ \new{with the threshold set to $t = .95$.} 

\myparagraph{Performance comparison}
The results of the exact RSMT solver
GeoSteiner 
are ground truth solutions. 
Since we could not compute GeoSteiner for large instances, we use FLUTE \cite{chu2007flute}, the SOTA heuristic solver, as the base for comparison of  approaches. Table \ref{tab: uniform-length} shows the performance comparison. Smaller values indicate better performance and negative values indicate superior performance to FLUTE, i.e., length is smaller than the output of FLUTE. All the FLUTE instances are run with $A = 18$, the setting which produces the highest-quality solutions. We also make comparisons against
REST \cite{liu2021rest} with $T = 8$, the best-performing setting claimed by the work. Batch sizes are set to 1.
Test pointsets are generated uniformly at random from a $10^4 \times 10^4$ grid. Reported values are averages over 100 pointsets.
Table \ref{tab: uniform-length} shows that \NNSteiner{} generalizes to large pointsets. (Recall that training pointsets have less than $n_0 = 180$ points.)  Furthermore, \NNSteiner{} outperforms FLUTE and REST for pointsets of size 500 or greater, with margin of outperformance increasing with pointset size. \NNSteiner{} has advantage over FLUTE for large problems, however, it performs best on small point sets relative to the exact solution. Also, Fig.\ \ref{fig:runtime} shows that the runtime of \NNSteiner{} scales approximately linearly.

\myparagraph{Generalization to Different Distributions}
We test the generalization of \NNSteiner{} to different pointset distributions. In particular, \NNSteiner{} is trained on uniformly distributed pointsets, but tested on pointsets with mixed normal and non-isotropic normal distributions.
Points are restricted to a $10^4 \times 10^4$ grid. The standard deviation of the non-isotropic normal distribution is taken to be $3000$ for both $x$ and $y$ and $1500$ for mixed-normal distribution. Means are distributed uniformly. The covariance matrix of the non-isotropic normal distribution is given by uniformly picking a correlation in $[-1, 1]$. For the mixed-normal distribution, we use a uniform mixture of 10 normal distributions.\footnote{{We do not generate mixed-normal distributions of 5000 points because the $10^4 \times 10^4$ granularity is too restrictive.}} Fig.\ \ref{fig:distribution} shows that \NNSteiner{} generalizes to different distributions, despite only being trained on uniformly distributed pointsets. 

\begin{figure}[!htbp]
    \centering
    \includegraphics[width =\columnwidth]{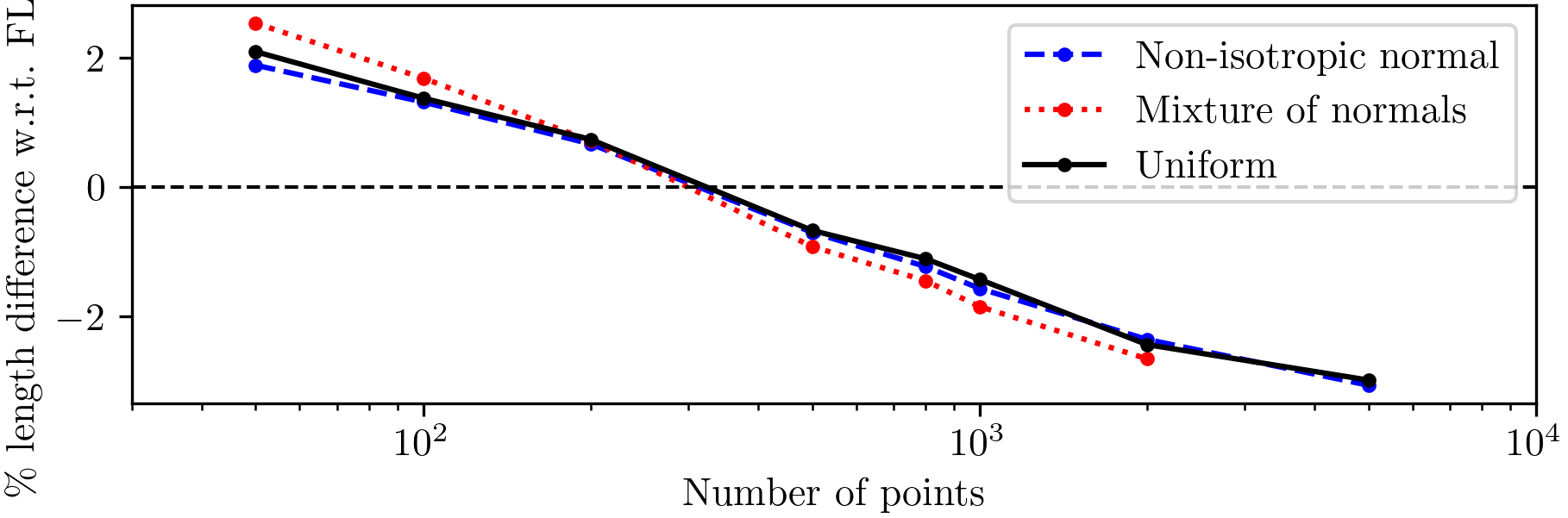}
    \caption{\small Performance of \NNSteiner{} on different distributions.}
    \label{fig:distribution}
\end{figure}

\myparagraph{Ablations}
We use an ablation study to show the importance of each component of our algorithm. To determine the impact of different components on performance as a whole, we evaluate the performance after removing each component. The components we consider are tree construction,  cell refinement, and subtree refinement (Fig.\ \ref{fig:nndp}). 
Tree construction produces an MST over $V \cup S$ where $S$ is found by portal retrieval. To evaluate performance without this component, we instead apply refinement to the MST over $V$ (here, edges between cells are preserved in cell refinement). Table \ref{tab: uniform-ablation} shows that each of these stages contributes to the overall performance, with refinement contributing the most. However, this observation does not undermine the importance of tree construction; if we apply refinement to the output of FLUTE \cite{chu2007flute}, also shown in Table \ref{tab: uniform-ablation}, we see limited improvement.  This suggests that the global topology produced by tree construction and cell refinement for large pointsets is superior to that produced by FLUTE.

\myparagraph{Hyperparameter choice}
We evaluate how $m$ and $k_b$ affect performance (Fig.\ \ref{fig:hyper}). Results show poor performance for small values of $k_b$, which we attribute to overfitting.  In theory, $(m,r)$-light trees are better approximations with greater values of $m$. However, for smaller instances, $m = 7$ yields similar performance as $m = 15$. This suggests that, for smaller instances, the fine granularity from larger $m$ is not needed to generate high-quality global tree topologies. Also, models with smaller $m$ may exhibit better performance as they are easier to train due to less classification imbalance. \new{ Experiments and discussion for threshold selection are  given in the supplemental material \cite{supp} }
\begin{figure}[!htbp]
    \centering
    \includegraphics[width =\columnwidth]{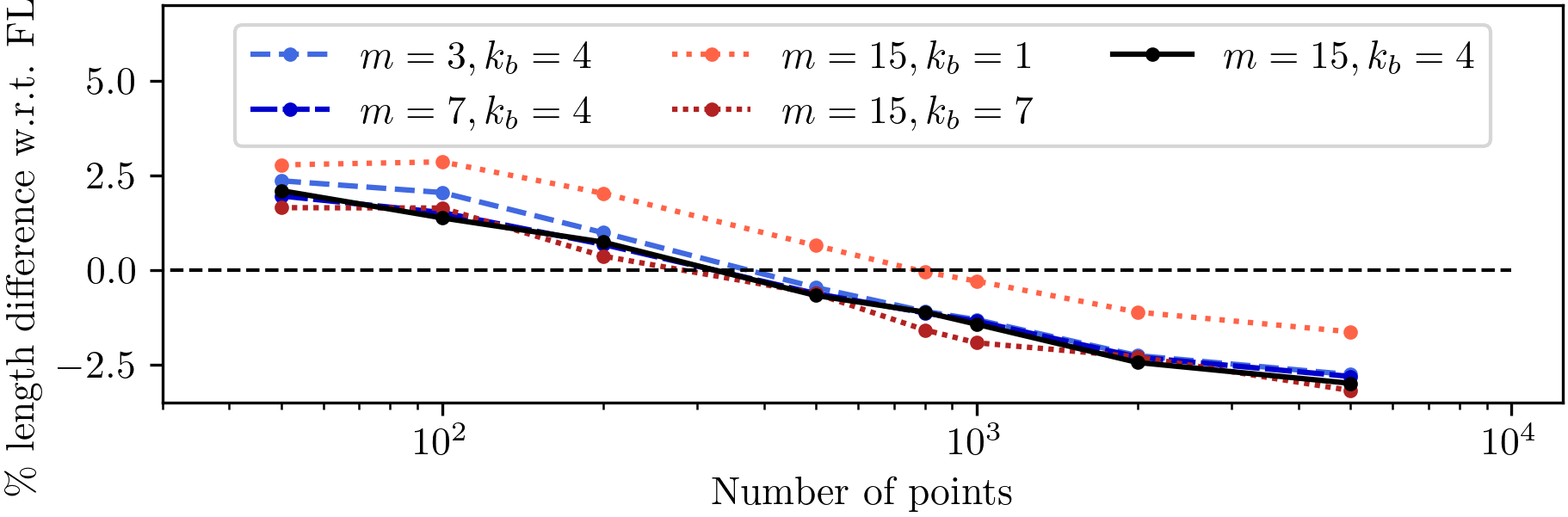}
    \caption{\small Dependence of \NNSteiner{} on $k_b$ and $m$.}
    \label{fig:hyper}
\end{figure}

\section{Conclusion}
We present a mixed neural-algorithmic framework, \NNSteiner{}, for solving large-scale RSMT problems. This is the first neural architecture that has the capacity to solve RSMT approximately.
In experiments, our framework shows generalization and scalability, and outperforms both heuristic algorithms and 
machine learning baselines on large instances. 

Our ongoing research pursues the following directions. 
First, \AroraAlg{} can solve higher-dimensional RSMT problems; this motivates us to generalize our framework to 3D IC designs. 
Second, the methodology behind \NNSteiner{} can be extended to compute obstacle-avoiding \RSMT{}s, which again have important applications in VLSI design.
%



\section*{Acknowledgments}
This work is partially supported by NSF under grants CCF-2112665, CCF-2217033 and CCF-2310411, and by DARPA IDEA HR0011-18-2-0032. The authors thank Anastasios Sidiropoulos for helpful discussions during initial stages of this project. We also thank Qi Zhao for early explorations and inputs to the early writing of this paper.
\bibliography{reference}

\vfill
\cleardoublepage

 \appendix

\section{Proof of Theorem \ref{th:structure}}

Our proof of Theorem \ref{th:structure} follows a similar argument as the one in \cite{arora1998polynomial} for geometric TSP.  We first
state the following lemma, which we use later to argue that any Steiner tree can be converted to an $(m,r)$-light Steiner tree with bounded cost increase. 
\begin{lemma}[Patching Lemma]
Let $\ell$ be any line segment of length $s$ and $T$ be a Steiner tree. The segment $\ell$ could be crossed by $T$ an arbitrary number of times. We can modify $T$ to a new Steiner tree that crosses the segment $\ell$  at most once while increasing the cost of the tree by at most $s$. 

\end{lemma}
\begin{proof}
Without loss of generality, assume the segment $\ell$ in Lemma A.1 is vertical. 
Then, the present lemma is shown by removing crossings one by one in a top-down manner, 
as follows. 
Let $z_1, \ldots, z_k$ be the set of intersection points between $T$ and the line 
segment $\ell$, sorted by decreasing $y$-coordinate. 
Consider $z_1$; imagine ``cutting" the tree $T$ at $z_1$, which breaks $T$ into two connected components; one containing a ``left copy" $z_1^-$ of $z_1$ and the other
containing a ``right copy" $z_1^+$ of $z_1$. One of these two components, say the one connecting $z_1^-$, is disconnected to the components containing $z_2$: we thus simply connect $z_1^-$ to the ``left copy" $z_2^-$ by a vertical edge. The resulting tree $T^{(1)}$ is still a valid Steiner tree. 
We repeat this until we finish processing all crossing points other than the last one, $z_k$. Overall, the total length of extra (vertical) edges we add is at most $s$. In the end only one crossing point (i.e, $z_k$) remains. 
\end{proof} 

\begin{lemma}
Grid the bounding box by putting a sequence of vertical and horizontal lines at unit distance from one another. If $\ell$ is one of these lines and $T$ is a Steiner tree, denote the number of times that $T$ crosses $\ell$ as $t(T, \ell)$. Then we have that
\begin{equation}
    \sum_{\ell:\text{vertical}} t(T, \ell) + \sum_{\ell:\text{horizontal}} t(T, \ell) \leq 2 \cdot \mathrm{cost}(T)
\end{equation}
\label{lm:crossing}
\end{lemma}
Lemma A.2 follows immediately from the fact that
an edge of $T$ with length $s$ can contribute at most $O(s)$ to the left hand side. 

Now let $T^*$ be a rectilinear Steiner minimum tree and suppose $(a, b)$ 
is picked randomly. 
We show how to convert $T^*$ to an $(m,r)$-light Steiner tree without increasing cost too much, via the following tree-transformation procedure.

First, similar to \cite{arora1998polynomial}, given any grid line $\ell$, we say $\ell$ has \emph{level $i$} in the shifted grid if it contains an edge of some level-$i$ square.
Note that a level-$i$ line is touched by $2^{i + 1}$ level-$(i + 1)$ cells, which partition it into $2^{i + 1}$ segments of length $L / 2^{i + 1}$. In general, for each $j > i$, line $\ell$ is also touched by $2^j$ level-$j$ cells. We refer to the portion of $\ell$ that lies in a level-$j$ cell as a level-$j$ segment. Our goal is to reduce the number of crossings in each level-$i$ segment to $r$ or less.

An \emph{overloaded segment of $\ell$} is one that the Steiner tree crosses more than $r$ times. For every segment at level $\log L - 1$ that is overloaded, we apply the Patching Lemma and reduce the crossings to $1$. Then we proceed to level $\log L - 2$ and apply the Patching Lemma to all overloaded segments at this level. We continue this procedure until no segments are overloaded at level $i$ for all $i$ from $\log L - 1$ down to $1$. By construction, the resulting new Steiner tree $T_1^*$ can cross each quadtree cell boundary at most $r$ times. 

Next, for each grid line $\ell$, consider its \emph{highest-level} $i$, which is the {\bf smallest} level that this line is at (intuitively, the smaller the level $i$ is, the higher it is up the quadtree). 
We then move all crossings of $T_1^*$ to their nearest portals along $\ell$ at this level. We denote the resulting Steiner tree by $\widehat{T}^*$, which by construction is $(m,r)$-light. 

What remains is to bound the cost of $\widehat{T}^*$ w.r.t. the optimal cost $\mathsf{OPT} = cost(T^*)$. 
First, we bound the cost increase due to crossing-reduction step of the tree-transformation procedure. 
Without loss of generality, let us fix  a vertical grid line $\ell$. 
We refer to the boundary segment of any level-$j$ quadtree cell as a \emph{level-$j$ segment}. 
The aforementioned crossing-reduction procedure processes $\ell$ from level $\log L -1$ to $0$ till no segment of any level from $\ell$ is overloaded. 
Let $X_{\ell, j}(b)$ be a random variable denoting the number of overloaded level-$j$ segments encountered in this procedure. Note $X_{\ell,j}(b)$ is determined by vertical shift $b$ (chosen randomly from $[0, L]$), which determines the location of crossings on $l$. 
We claim that for every $b$, we have with probability 1
\begin{equation}
\sum_{j \geq 0}X_{\ell,j}(b) \leq \frac{t(T^*, \ell)}{r}.
\label{eq:cr}
\end{equation}
This is because the optimal input Steiner tree $T^*$ crossed grid line $\ell$ only $t(T^*, \ell)$ times, and each application of the Patching Lemma counted on the left hand side of Eqn.\ (\ref{eq:cr}) replaces at least $r + 1$ number of crossings by 1, thus eliminating at least $r$ crossings each time. 

Since a level-$j$ segment has length $L / 2^j$, the total extra cost incurred during the transformation procedure (as we described above) is at most $\sum_{j \geq 1}X_{\ell,j}(b) \frac{L}{2^{j}} $
by the Patching Lemma.

Now we want to bound the total cost-increase produced by all segments across all levels that grid line $\ell$ can generate. 
The actual cost increase at $\ell$ during the crossing-reduction process depends on the level of $\ell$, which is determined by horizontal shift $a$ (which was chosen randomly from $[0, L]$ during our random shift step). 
If the highest-level of $\ell$ is $i$, 
then the cost increase can be upper-bounded by $\sum_{j \geq i+1}X_{\ell,j}(b) \frac{L}{2^{j-1}}$. As $a$ is chosen randomly in $[0, L]$, the probability that $i$ is the highest-level of line $\ell$ is at most $2^{i + 1} / L$. 
Let $Y_{\ell,a}$ denote the total cost-increase due to changes to $\ell$ when the horizontal shift is $a$.  We can now bound the expectation of $Y_{\ell, a}$; in particular, for any vertical shift $b$: 
\begin{equation}
\begin{aligned}
    E_a[Y_{\ell,a}] & \leq \sum_{i \geq 1} \frac{2^{i + 1}}{L} \cdot \sum_{j \geq i + 1} X_{\ell,j}(b) \frac{L}{2^{j}} \\
    & =\sum_{j \geq 1} \frac{X_{\ell,j}(b)}{2^{j}} \sum_{i \leq j-1}2^i\\
    &\leq \sum_{j \geq 1} X_{\ell,j}(b) \\
    &\leq \frac{t(T^*,\ell)}{r}
\end{aligned}
\end{equation}
The expectation of cost-increase for all  vertical lines is therefore $E_a[\sum_{\ell~\text{is vertical}} Y_{\ell,a}] = \sum_{\ell~ \text{is vertical}} \frac{t(T^*,\ell)}{r}$. 
A symmetric argument can be used to bound the expected cost-increase for all horizontal lines by $\sum_{\ell~\text{is horizontal}} \frac{t(T^*,\ell)}{r}$. It then follows from Lemma \ref{lm:crossing} that the expected cost-increase incurred by all grid lines is bounded above by
\begin{equation}
    \sum_{\ell~ \text{is vertical}} \frac{t(T^*,\ell)}{r} + \sum_{\ell~ \text{is horizontal}} \frac{t(T^*,\ell)}{r} \leq \frac{2 \mathrm{cost}(T^*)}{r}.
\end{equation}

Finally, we also need to bound the cost increase when moving crossings to their nearest portals. If a grid line $\ell$ has highest-level $i$, the distance between each of the $t(T^*, \ell)$ number of crossings and its nearest portal is at most $\frac{L}{2^{i+1}m}$. Instead of actually moving a crossing to a portal, we break each edge at the crossing and add two line segments (on each side of $\ell$) connecting the portal to the origional crossing. Thus the expected increase for moving every crossing in $\ell$ to nearest portals can be bounded by $\sum_{i=1}^{\log L} \frac{2^i}{L} t(T^*, \ell) \cdot \frac{L}{2^{i + 1}m} \cdot 2 = \frac{t(T^*, \ell)\log L}{m}$.

Using Lemma \ref{lm:crossing}, the total expected cost increase for all crossings in all lines is bounded from above by
\begin{equation}
     \sum_{\ell} \frac{t(T^*, \ell) \log L}{m} \leq \frac{2\log L}{m} \mathrm{cost}(T^*)
\end{equation}

Denote by $T_{a,b,m,r}$ the Steiner tree obtained from $T^*$ by $(a,b)$-shift, followed by the aforementioned transformation procedure (consisting of crossing-reduction step, and the moving of crossings to portals). Putting everything together, the expected cost $E[\mathrm{cost}(T_{a,b,m,r})]$ can be bounded by: 
$E[\mathrm{cost}(T_{a,b,m,r})] \leq (1 + \frac{2}{r} + O(\frac{2\log L}{m}))\mathrm{cost}(T^*).$
Using Markov's inequality, with probability at least $1/2$, the cost of the best $(m,r)-$light Steiner tree for the shifted dissection is at most $(1+\frac{4}{r} + O(\frac{4\log L}{m}))-$OPT. This finishes the proof of Theorem \ref{th:structure}. 
\cleardoublepage
\onecolumn
\section{Additional experimental results}

In this section we give tables containing the data displayed in the figures of the main text and the corresponding runtimes for these experiments. 
All the runtimes are totals over 100 pointsets in seconds, and all other values are averages over 100 pointsets. 

\begin{table}[!h]
\small
\centering
\begin{tabular}{|l|r|r|r|r|r|r|r|r|}
\hline
Number of points & 50            & 100           & 200            & 500            & 800            & 1000           & 2000                   & 5000                   \\ \hline
NN-Steiner       & 9.23          & 19.82         & 36.68          & 95.97          & 147.53         & 186.65         & 383.45                 & 1021.00                \\ \hline
REST (T=8)        & 59.69         & 118.13        & 230.56         & 586.91         & 927.63         & 1174.98        & 2316.37                & 5707.48                \\ \hline
FLUTE (A=18)      & 3.68          & 6.30          & 13.85          & \textbf{42.23} & \textbf{74.65} & \textbf{97.25} & \textbf{218.35}        & \textbf{613.77}        \\ \hline
Geosteiner       & \textbf{0.83} & \textbf{2.74} & \textbf{10.27} & 86.32          & 534.99         & 5998.50        & --& -- \\ \hline
\end{tabular}
\caption{\small Comparison of runtimes in seconds for uniformly distributed pointsets, plotted in Fig.\ \ref{fig:runtime}. }
\end{table}

\begin{table}[!h]
\small
\centering
\begin{tabular}{|l|r|r|r|r|r|r|r|r|}
\hline
Number of points & 50   & 100   & 200   & 500   & 800    & 1000   & 2000   & 5000    \\ \hline
NN tree construction + refinement (\NNSteiner{})       & 9.23 & 19.82 & 36.68 & 95.97 & 147.53 & 186.65 & 383.45 & 1021.00 \\ \hline
MST of $V$ + refinement & 2.96 & 6.30  & 11.88 & 30.77 & 48.58  & 62.04  & 125.55 & 332.46  \\ \hline
NN tree construction + cell refinement  & 8.90 & 19.38 & 35.83 & 91.70 & 139.82 & 178.42 & 367.94 & 989.00  \\ \hline
NN tree construction + subtree refinement  & 8.80 & 19.06 & 35.30 & 91.79 & 137.49 & 175.91 & 362.50 & 958.87  \\ \hline
\end{tabular}
\caption{\small Ablation experiment runtimes in seconds for uniformly distributed pointsets. }
\end{table}

\begin{table}[!h]
\small
\centering
\begin{tabular}{|l|r|r|r|r|r|r|r|r|}
\hline
Number of points     & 50   & 100  & 200  & 500   & 800   & 1000  & 2000  & 5000                    \\ \hline
Uniform              & 2.10 & 1.38 & 0.74 & -0.67 & -1.11 & -1.43 & -2.44 & -2.99                   \\ \hline
Mixture              & 2.54 & 1.69 & 0.72 & -0.92 & -1.45 & -1.85 & -2.66 & \multicolumn{1}{l|}{--} \\ \hline
Non-isotropic normal & 1.89 & 1.32 & 0.67 & -0.70 & -1.23 & -1.57 & -2.36 & -3.07                   \\ \hline
\end{tabular}
\caption{\small Performance of \NNSteiner{} on different distributions (average percent length difference compared to FLUTE), plotted in Fig.\ \ref{fig:distribution}. }
\end{table}

\begin{table}[!h]
\small
\centering
\begin{tabular}{|l|r|r|r|r|r|r|r|r|}
\hline
Number of points     & 50    & 100   & 200   & 500    & 800    & 1000   & 2000   & 5000                    \\ \hline
Uniform              & 9.23  & 19.82 & 36.68 & 95.97 & 147.53 & 186.65 & 383.45 & 1021.00                 \\ \hline
Mixture              & 10.45 & 20.95 & 40.64 & 97.96 & 153.89 & 191.43 & 379.80 & \multicolumn{1}{l|}{--} \\ \hline
Non-isotropic normal & 9.45  & 19.74 & 38.95 & 95.31 & 151.64 & 187.16 & 378.56 & 993.56                  \\ \hline
\end{tabular}
\caption{\small Runtimes of \NNSteiner{} in seconds on different distributions. }
\end{table}

\begin{table}[!h]
\small
\centering
\begin{tabular}{|l|r|r|r|r|r|r|r|r|}
\hline
Number of points & 50   & 100  & 200  & 500   & 800   & 1000  & 2000  & 5000  \\ \hline
$m=15$      & 2.10 & 1.38 & 0.74 & -0.67 & -1.11 & -1.43 & -2.44 & -2.99 \\ \hline
$m=7$ & 1.96 & 1.52 & 0.68 & -0.62 & -1.13 & -1.35 & -2.30 & -2.81 \\ \hline
$m=3$ & 2.36 & 2.05 & 0.99 & -0.46 & -1.09 & -1.31 & -2.26 & -2.76 \\ \hline
\end{tabular}
\caption{\small Performance of \NNSteiner{} as a function of $m$ (average percent length difference compared to FLUTE), plotted in Fig.\ \ref{fig:hyper}.}
\end{table}

\begin{table}[!h]
\small
\centering
\begin{tabular}{|l|r|r|r|r|r|r|r|r|}
\hline
Number of points & 50   & 100   & 200   & 500   & 800    & 1000   & 2000   & 5000    \\ \hline
$m=15$      & 9.23 & 19.82 & 36.68 & 95.97 & 147.53 & 186.65 & 383.45 & 1021.00 \\ \hline
$m=7$ & 7.15 & 14.90 & 27.61 & 72.07 & 110.55 & 139.72 & 285.94 & 763.76  \\ \hline
$m=3$ & 5.77 & 12.38 & 22.72 & 58.81 & 90.34  & 115.50 & 235.79 & 630.96  \\ \hline
\end{tabular}
\caption{\small Runtimes of \NNSteiner{} in seconds as a function of $m$. }
\end{table}

\begin{table}[!h]
\small
\centering
\begin{tabular}{|l|r|r|r|r|r|r|r|r|}
\hline
Number of points & 50   & 100  & 200  & 500   & 800   & 1000  & 2000  & 5000  \\ \hline
$k_b=1$ & 2.78 & 2.86 & 2.03 & 0.65  & -0.05 & -0.29 & -1.11 & -1.63 \\ \hline
$k_b=4$       & 2.10 & 1.38 & 0.74 & -0.67 & -1.11 & -1.43 & -2.44 & -2.99 \\ \hline
$k_b=7$ & 1.65 & 1.64 & 0.37 & -0.62 & -1.58 & -1.92 & -2.29 & -3.18 \\ \hline
\end{tabular}\caption{\small Performance of \NNSteiner{} as a function of $k_b$ (average percent length difference compared to FLUTE), plotted in Fig.\ \ref{fig:hyper}.}
\end{table}

\clearpage
\begin{table}[h!]
\small
\centering
\begin{tabular}{|l|r|r|r|r|r|r|r|r|}
\hline
Number of points & 50    & 100   & 200    & 500    & 800    & 1000   & 2000    & 5000    \\ \hline
$k_b=1$ & 35.17 & 70.97 & 136.38 & 337.83 & 545.86 & 674.61 & 1386.92 & 3569.93 \\ \hline
$k_b=4$       & 9.23  & 19.82 & 36.68  & 95.97  & 147.53 & 186.65 & 383.45  & 1021.00 \\ \hline
$k_b=7$ & 5.89  & 11.52 & 24.33  & 62.90  & 97.79  & 114.05 & 251.53  & 577.57  \\ \hline
\end{tabular}
\caption{\small Runtimes of \NNSteiner{} in seconds as a function of $k_b$. }
\end{table}

\section{Threshold selection} \label{sec:threshold}
We conducted experiments evaluating performance of \NNSteiner{} across a range of thresholds and found that a threshold of .95 yields superior results. The threshold results are shown in Fig.\ \ref{fig:thresh} and Table \ref{tab: thresh}.
\begin{figure}[H] 
    \centering

    \includegraphics[width = 10 cm]{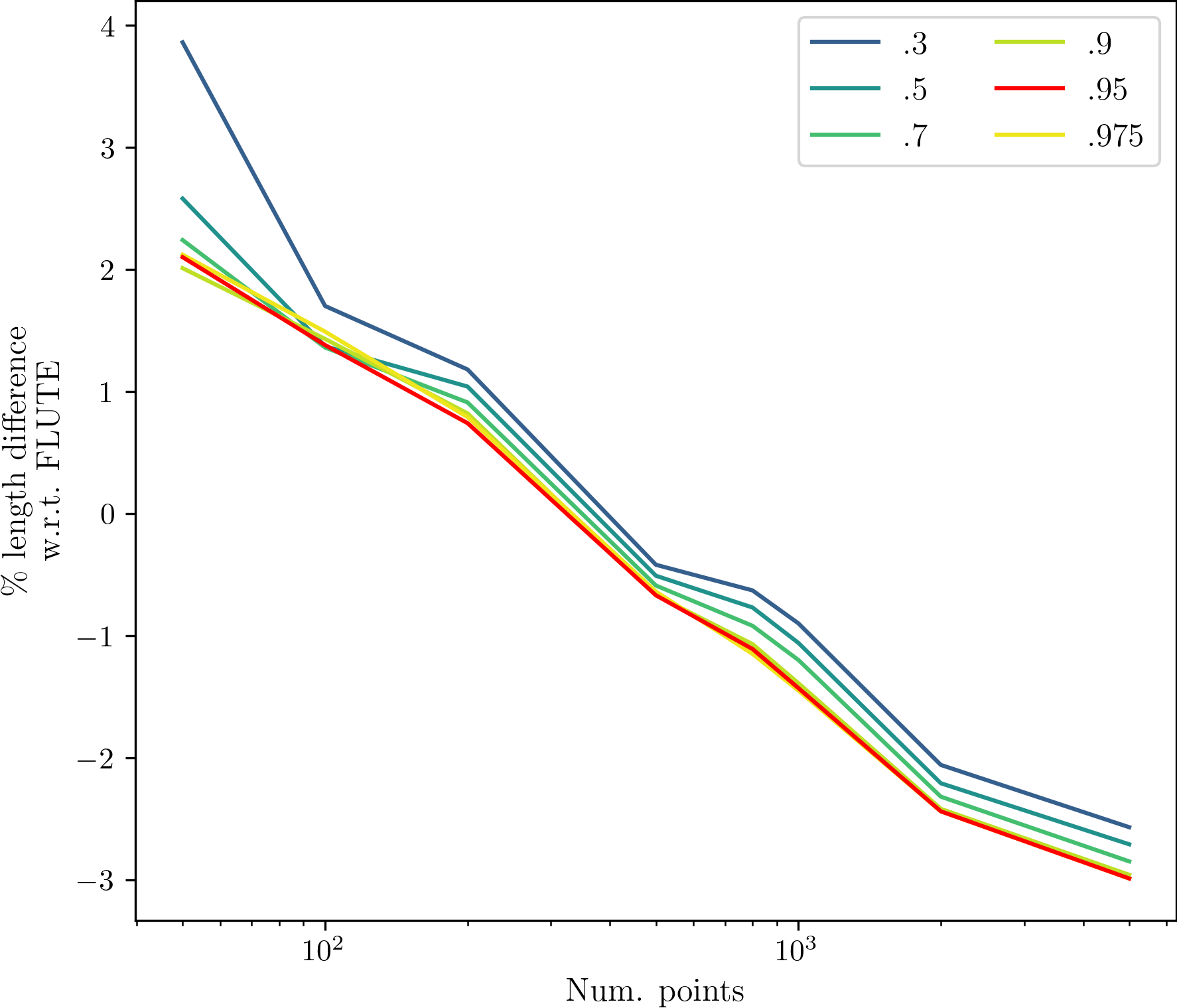}
        \caption{Dependence of \NNSteiner{} on threshold.}\label{fig:thresh}
\end{figure}

\begin{table}[H]
\centering
\begin{tabular}{|r|r|r|r|r|r|r|r|r|}
\hline
\multicolumn{1}{|l|}{Threshold \textbackslash \ Number of points} & 50   & 100  & 200  & 500   & 800   & 1000  & 2000  & 5000  \\ \hline
0.1                                                      & 5.48 & 2.46 & 1.78 & -0.05 & -0.15 & -0.50 & -1.75 & -2.26 \\ \hline
0.2                                                      & 5.04 & 2.01 & 1.46 & -0.33 & -0.45 & -0.74 & -1.94 & -2.45 \\ \hline
0.3                                                      & 3.86 & 1.70 & 1.18 & -0.42 & -0.63 & -0.90 & -2.06 & -2.57 \\ \hline
0.4                                                      & 2.99 & 1.44 & 1.16 & -0.48 & -0.70 & -0.99 & -2.14 & -2.65 \\ \hline
0.5                                                   & 2.58 & 1.37 & 1.04 & -0.51 & -0.77 & -1.06 & -2.21 & -2.71 \\ \hline
0.6                                                      & 2.37 & 1.34 & 0.95 & -0.54 & -0.85 & -1.12 & -2.26 & -2.78 \\ \hline
0.7                                                      & 2.24 & 1.36 & 0.91 & -0.59 & -0.92 & -1.20 & -2.32 & -2.85 \\ \hline
0.8                                                      & 1.98 & 1.39 & 0.84 & -0.63 & -0.97 & -1.30 & -2.37 & -2.91 \\ \hline
0.9                                                      & 2.01 & 1.43 & 0.82 & -0.67 & -1.07 & -1.39 & -2.42 & -2.96 \\ \hline
0.95                                                     & 2.10 & 1.38 & 0.74 & -0.67 & -1.11 & -1.43 & -2.44 & -2.99 \\ \hline
0.975                                                    & 2.12 & 1.49 & 0.79 & -0.64 & -1.15 & -1.45 & -2.44 & -2.99 \\ \hline
1                                                        & 2.53 & 2.44 & 1.36 & -0.02 & -0.78 & -1.08 & -1.78 & -2.49 \\ \hline
\end{tabular}
\caption{Dependence of \NNSteiner{} on threshold.} \label{tab: thresh}
\end{table}
\end{document}